\documentclass[letterpaper, 10 pt, conference]{ieeeconf}  

\IEEEoverridecommandlockouts                              

\makeatletter

\let\proof\@undefined
\let\endproof\@undefined
\makeatother

\usepackage{amsmath, lmodern, bm} 
\usepackage{amssymb} 
\usepackage{amsthm}
\usepackage{dsfont}

\usepackage{color, colortbl}
\usepackage[bottom]{footmisc}

\usepackage{float}
\usepackage{subfigure}
\usepackage{graphicx}
\usepackage{caption}
\usepackage[ruled]{algorithm2e}

\usepackage{cite}

\usepackage{tikz}
\usetikzlibrary{shapes,arrows,positioning,calc,backgrounds}

\newtheorem{theorem}{Theorem}

\newcommand{\vect}[1]{\mathbf{#1}}

\DeclareMathOperator*{\argmin}{arg\,min}

\definecolor{gray}{rgb}{0.5, 0.5, 0.5}

\definecolor{gray}{rgb}{0.5, 0.5, 0.5}
\definecolor{lightgray}{rgb}{0.83, 0.83, 0.83}

\newcommand{\nlenv}[1]{\begin{tabular}{c} 
    #1
  \end{tabular}
}

\newtheorem*{problem}{Problem}

\usepackage[normalem]{ulem}

\title{\LARGE \bf
Probabilistic Radio-Visual Active Sensing \\for Search and Tracking}

\author{Luca~Varotto,
        Angelo~Cenedese,
        and~Andrea~Cavallaro%
\thanks{L. Varotto and A. Cenedese are with the Department of Information Engineering, University of Padova, Italy.
          Corresponding author: {\tt\small luca.varotto.5@phd.unipd.it}.}%
\thanks{A. Cavallaro is with the Centre for Intelligent Sensing, Queen Mary University of London, U.K.}
\thanks{This work was partially supported by the Department of Information Engineering under the BIRD-SEED TSTARK project and by the Gini Foundation of the University of Padova.}
}%

\begin{document}

\maketitle
\thispagestyle{empty}
\pagestyle{empty}

\begin{abstract}
Active Search and Tracking for search and rescue missions or collaborative mobile robotics relies on the actuation of a sensing platform to detect and localize a target. In this paper we focus on visually detecting a radio-emitting target with an aerial robot equipped with a radio receiver and a camera. Visual-based tracking provides high accuracy, but the directionality of the sensing domain may require long search times before detecting the target. Conversely, radio signals have larger coverage, but lower tracking accuracy. Thus, we design a Recursive Bayesian Estimation scheme that uses camera observations to refine radio measurements. To regulate the camera pose, we design an optimal controller whose cost function is built upon a probabilistic map. Theoretical results support the proposed algorithm, while numerical analyses show higher robustness and efficiency with respect to visual and radio-only baselines.  
\end{abstract}


\section{Introduction}\label{sec:intro}

{Active sensing} consists of controlling the sensor state to gather (more) informative observations~\cite{radmard2017active}
or to accomplish a task (e.g., find a target within a certain time budget~\cite{MTS_Rinner}). 
This control framework has been largely used to solve autonomous target search and tracking~\cite{shahidian2017single},
often relying on probabilistic approaches~\cite{detection_tracking_survey}:
data from onboard sensors and Recursive Bayesian Estimation (RBE) schemes~\cite{smith2013MonteCarlo}
are used to generate a probabilistic map (also known as belief map),
encoding the knowledge about potential target locations. The control problem is then cast as the optimization of a suitable objective function built upon the probabilistic map (e.g., time to detection~\cite{MTS_Rinner}, estimate uncertainty~\cite{shahidian2017single},
distance to the target~\cite{hasanzade2018rf}). 
Stochastic motion
and observation models~\cite{radmard2017active} account for the 
uncertainties on target dynamics and on the perception process, and allow to treat no-detection observations~\cite{negative_information}.
For these reasons, probabilistic approaches are suitable for real-life scenarios, which are also characterized by energy costs associated to the movement of the active sensing platform~\cite{liu2018energy}. 


$ $

\noindent \textbf{Related works -}
Typical modalities for active sensing include vision, audio and radio~\cite{radmard2017active,haubner2019active,shahidian2017single}. Visual-based tracking provides high accuracy~\cite{aghajan2009multi}
and does not require the target to use an emitting device. 
Occlusions and Field of View (FoV) directionality~\cite{aghajan2009multi}
limit the range, applicability and success of camera-only platforms~\cite{detection_tracking_survey}, especially for applications where time of detection is critical (e.g., search and rescue missions~\cite{SAR_radio}). To collect measurements on wider ranges, and reduce the duration of the search phase, 
acoustic~\cite{haubner2019active}
or radio-frequency (RF)~\cite{shahidian2017single} signals can be used.  
Despite the high localization accuracy of acoustic signals~\cite{haubner2019active}, sound pollution and
extra hardware requirements (e.g., microphone arrays)
are drawbacks of this technology~\cite{zafari2019survey}.  Conversely, RF signals are energy efficient, have large 
reception ranges ($\sim100 \; [m]$), and low hardware requirements, since only a receiver is needed; moreover, the Received Signal Strength Indicator (RSSI) is extracted from standard data packet traffic~\cite{zanella2016best}. 
For these reasons, RSSI-based localization systems widely appear in the literature and in commercial applications, despite environmental interference (e.g., cluttering and multi-path distortions) often limits their accuracy~\cite{zanella2016best}. 
Multi-modal sensor fusion techniques have been shown to overcome the inadequacies of uni-modal approaches, being more robust and reliable~\cite{DeepRL_gazeControl}. 

$ $

\noindent \textbf{Contributions -} This paper exploits the complementary benefits of radio and visual cues for visually detecting a radio-emitting target with an aerial robot, equipped with a radio receiver and a Pan-Tilt (PT) camera. We formulate the control problem within a probabilistic active sensing framework, where 
camera measurements refine radio ones within a RBE scheme, used to keep the map updated. The fusion of RF and camera sensor data for target search and tracking is an open problem.
To the best of authors' knowledge, this is the first attempt to combine radio and visual measurements 
within a single-platform probabilistic active sensing framework. 
Furthermore, unlike existing solutions operating on limited control spaces (e.g., platform position~\cite{shahidian2017single}
or camera orientation~\cite{DeepRL_gazeControl}), 
we propose a gradient-based optimal control, defined on a continuous space comprising both platform position and camera orientation. Theoretical and numerical analyses are provided to validate the effectiveness of the proposed algorithm. 
%
%
What emerges is that bi-modality is proven to increase the target localization accuracy; this, together with the availability of an integrated high-dimensional control space, leads to higher detection success rates, as well as superior time and energy efficiency with respect to radio-only or and vision-only counterparts.


\section{Problem statement}\label{sec:problem_formulation}

Fig. \ref{fig:scenario} shows the main elements of the problem scenario, namely the target and the sensing platform\footnote{Bold letters indicate (column) vectors, if lowercase, matrices otherwise. $\vect{I}_n$ is the $n$-dimensional identity matrix, while $\vect{0}_n$ is the zero vector of dimension $n$.
Regarding the statistical distributions, $\chi^2(n)$ denotes the chi-squared distribution with $n$ degrees of freedom, and $\mathcal{N}(x|\mu,\sigma^2)$ is the Gaussian distribution over the random variable $x$ with expectation $\mu$ and variance $\sigma^2$.
With the shorthand notation $z_{t_0:t_1}$ we indicate the sequence 
$\left\lbrace z_k \right\rbrace_{k=t_0}^{t_1}$.
The Euclidean distance between vectors $\vect{a}, \vect{b} \in \mathbb{R}^n$ is denoted as $d(\vect{a},\vect{b})$.
The orthogonal projection of $\vect{a} \in \mathbb{R}^n$ onto the plane $\Pi$ is $\vect{a}_\Pi$. The \mbox{$\ell$-th} entry of vector $\vect{a}$ is denoted as $\vect{a}(\ell)$.}.  

\noindent \textbf{Target - }The radio-emitting target
moves on a planar environment \mbox{$\Pi \subset \mathbb{R}^2$}, according to a (possibly) non-linear stochastic Markovian state transition model~\cite{radmard2017active} 
\begin{equation}\label{eq:target_dyn}
\vect{p}_{t+1} = f(\vect{p}_{t},\bm{\eta}_{t}).
\end{equation}
where $\vect{p}_t \in \Pi$ is the target position at time $t$, referred to the global 3D reference frame $\mathcal{F}_0$; when expressed in $\mathbb{R}^3$, it is referred as $\vect{p}_t^+ = [ \, \vect{p}_t^\top \quad 0 \,]^\top$.
The uncertainty on the underlying target movements are captured by 
the stochastic process noise $\bm{\eta}_t$. 
The probabilistic form of \eqref{eq:target_dyn}, namely $p(\vect{p}_{t+1}|\vect{p}_{t})$, is known as process model~\cite{radmard2017active}.  

\noindent \textbf{Sensing platform - }The sensing platform is an unmanned aerial vehicle (UAV), equipped with an omnidirectional radio receiver and a PT camera 
endowed with processing capabilities
and a real-time target detector~\cite{yolo}.
The state of the platform is the camera pose
\begin{equation}\label{eq:platform_state}
\begin{split}
    & \vect{s}_t = \begin{bmatrix} \vect{c}_t^\top & \bm{\psi}_t^\top \end{bmatrix}^\top, \\
    & \vect{c}_t \in \mathbb{R}^3; \quad \bm{\psi}_t = \begin{bmatrix} \alpha_t & \beta_t \end{bmatrix}^\top \in [-\pi/2+\theta,\pi/2-\theta]^2.
\end{split}
\end{equation}
The UAV position $\vect{c}_t$ is referred to $\mathcal{F}_0$, it is supposed to coincide with the camera focal point and its altitude is fixed (i.e., non-controllable); $\alpha_t$ (resp. $\beta_t$) is the pan (resp. tilt) angle w.r.t. the camera inertial reference frame;
$\theta$ is the half-angle of view. The state follows a linear deterministic Markovian transition model~\cite{radmard2017active}
\begin{equation}\label{eq:sensor_dynamics}
    \vect{s}_{t+1} = \vect{s}_t + \vect{u}_t; \; \vect{u}_t = \begin{bmatrix} \vect{u}_{\vect{c},t}^\top \quad \vect{u}_{\bm{\psi},t}^\top \end{bmatrix}^\top \in \mathcal{A}
\end{equation}
where $\mathcal{A}$ is the control space.
It comprises all possible control inputs that can be applied to the platform to regulate position and attitude. In particular, being the UAV altitude fixed, we focus on a planar control $\vect{u}_{\vect{c}_\Pi,t}$, acting on the projection $\vect{c}_{\Pi,t}$. 
Inspired by real-life scenarios, the UAV movements are considered energy-consuming with a linear dependence on the flying distance~\cite{liu2018energy}, that is
\begin{equation}\label{eq:energy_model}
    \Delta E_t = d(\vect{c}_{t+1},\vect{c}_{t}),
\end{equation}
where $\Delta E_t$ is the energy used to move the platform from $\vect{c}_{t}$ to $\vect{c}_{t+1}$. 
The total available energy is denoted as $E_{tot}$.  

\begin{figure}[t]
\centering
\includegraphics[scale=0.25
]{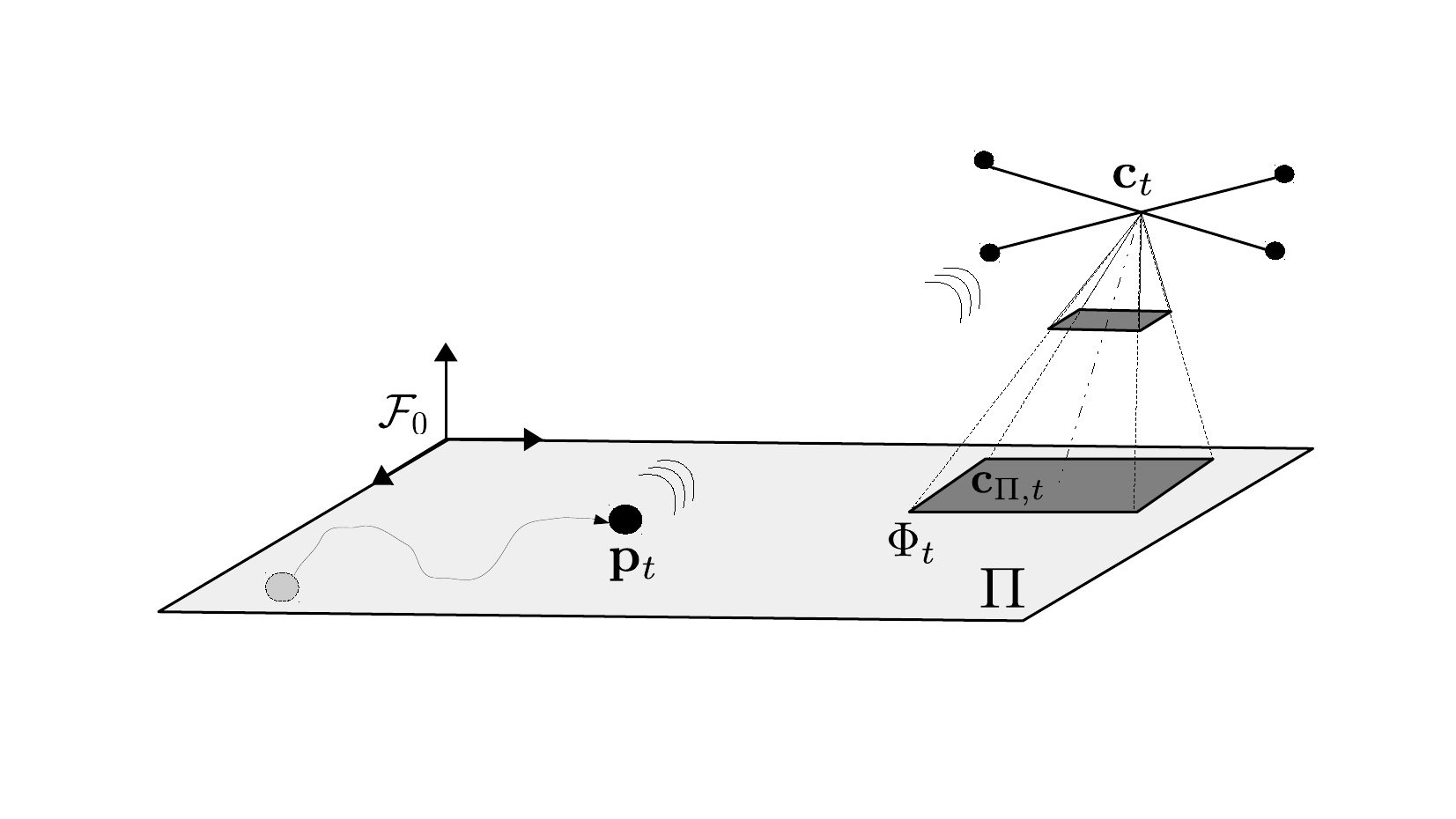}
\vspace{-0.9cm}
\caption{Problem scenario. A target moves in a planar environment and establishes a radio signal communication with a camera-embedded UAV. The objective is to control the camera pose so that the target is visually detected.}
\label{fig:scenario}
\vspace{-0.5cm}
\end{figure}

Motivated by the long reception ranges of radio signals~\cite{zanella2016best}, we suppose the target to be always within the range of the platform receiver and, from received data packets, the RSSI value $r_t\in\mathbb{R}$ is extracted. This is related to the platform-target distance $d(\vect{c}_t,\vect{p}_t^+)$ according to the log-distance path loss model~\cite{zanella2016best} 
\begin{equation}\label{eq:path_loss_model}
r_t = \kappa - 10 n\log_{10} \left( d(\vect{c}_t,\vect{p}_t^+) \right).
\end{equation}
The parameter $\kappa$ is the RSSI at a reference distance (e.g., $1\;[m]$), while $n$ is the attenuation gain; both $\kappa$ and $n$ are estimated via offline calibration procedures~\cite{zanella2016best}.
The radio observation model is
\begin{equation}\label{eq:obs_Rx}
z_{\text{RF},t} =
\begin{cases} 
r_t + v_{\text{RF},t} , & t = M T_{\text{RF}}, \; M \in \mathbb{N}\\
\emptyset, & \text{otherwise}
\end{cases}
\end{equation}
where $T_{\text{RF}}$ is the receiver sampling interval, \mbox{$v_{\text{RF},t} \sim \mathcal{N}\left(v| 0,\sigma_{\text{RF}}^2 \right)$} is the noise in RSSI data, $\emptyset$ is 
a measurement without target information.

The camera observation model follows the projection perspective geometry~\cite{aghajan2009multi}
\begin{equation}\label{eq:obs_c}
\vect{z}_{c,t} =  
\begin{cases}  
\vect{P}(\vect{s}_t)\widetilde{\vect{p}}_t + \vect{v}_{c,t}, & D_t =1 \text{ and } t = N T_a, \; N \in \mathbb{N}\\
\emptyset, & \text{otherwise}
\end{cases}
\end{equation}
where a successful target detection is indicated by the value $1$ of the binary variable $D_t$,
$\vect{v}_t \sim \mathcal{N}\left(\vect{v}| \vect{0}_2,\bm{\Sigma}_{c} \right)$ is the noise of camera observations, $\widetilde{\vect{p}}_t$ is the homogeneous representation of $\vect{p}_t^+$ 
and 
%
%
\begin{equation}
\vect{P}(\vect{s}_t) = \begin{bmatrix} \vect{I}_2 & \vect{0}_2 \end{bmatrix} \vect{K} \begin{bmatrix} \vect{R}(\bm{\psi_t}) & \vect{c}_t \end{bmatrix} \in \mathbb{R}^{2 \times 4}
\end{equation}
is the camera projection matrix that maps $\widetilde{\vect{p}}_t$ onto the image plane $\mathcal{I}$. $\vect{P}(\vect{s}_t)$ depends on $\vect{K} \in \mathbb{R}^{3 \times 3}$, the matrix of intrinsic parameters, and $\vect{R}(\bm{\psi_t})$, the camera rotation matrix w.r.t. $\mathcal{F}_0$. The camera frame rate $T_a$ satisfies
\begin{equation}\label{eq:sampling_time_relation}
    T_{\text{RF}} = \nu T_a, \; \nu>1
\end{equation}
since radio reception is typically characterized by higher sample rates than  cameras~\cite{zhuang2016smartphone,vollmer2011high}. Without any loss of generality, we consider a normalized frame rate ($T_a=1$). 

\begin{problem} With this formalism, the \emph{visual target detection problem} can be formulated as the control of the platform state $\vect{s}_t$ (through $\vect{u}_t$) to realize event $D_t=1$.
\end{problem}

\section{Methodology}\label{sec:method}

To solve the problem defined in Sec. \ref{sec:problem_formulation}, a probabilistic bi-modal active sensing approach is proposed. As shown in
Fig.~\ref{fig:pipeline}, radio-visual measurements are aggregated into a single likelihood function, which is used to update the belief map through a RBE scheme. The map is then fed into the controller to regulate the platform movements.

\noindent \textbf{Probabilistic map -} 
Given the observations $\vect{z}_{1:t}$, RBE provides a two-stage procedure to recursively update the target belief state, namely the posterior distribution $p(\vect{p}_t| \vect{z}_{1:t})$. The prediction stage involves using the process model \eqref{eq:target_dyn} to obtain the prior of the target position via the \mbox{Chapman-Kolmogorov} equation~\cite{radmard2017active}.
As a new observation $\vect{z}_{t}$ becomes available, the Bayes rule~\cite{smith2013MonteCarlo} updates the target belief state.
In this work, RBE is implemented through particle filtering~\cite{smith2013MonteCarlo}. 
The density $p(\vect{p}_t|\vect{z}_{1:t})$ is approximated with a sum of $N_s$ Dirac functions centered in the particles $\{ \vect{p}_t^{(i)} \}_{i=1}^{N_s}$, that is
\begin{equation}\label{eq:posterior_PF}
p(\vect{p}_t|\vect{z}_{1:t}) \approx \sum_{i=1}^{N_s} w_t^{(i)} \delta \left( \vect{p}_t - \vect{p}_t^{(i)} \right),
\end{equation}
where $w_t^{(i)}$ is the weight of particle $\vect{p}_t^{(i)}$ and it holds
\begin{subequations}\label{eq:PF}
\begin{align}
& \vect{p}_t^{(i)} = f\left( \vect{p}_{t-1}^{(i)},\bm{\eta}_{t-1} \right) \; \text{PREDICTION} \label{eq:prediction} \\
& w_t^{(i)} \propto w_{t-1}^{(i)}p\left(\vect{z}_t|\vect{p}_t^{(i)}\right) \; \text{UPDATE} \label{eq:update}
\end{align}
\end{subequations}

\tikzset{
block/.style = {draw, fill=white, rectangle, minimum height=3em, minimum width=3em},
block_transp/.style = {rectangle, minimum height=3em, minimum width=3em},
sum/.style= {draw, fill=white, circle, node distance=1cm}}
\begin{figure}[t!]
\center
\begin{tikzpicture}[auto, node distance=3cm,>=latex',scale=0.6, transform shape]

\node [block] (sensing) { \nlenv{sensing\\ units}};
\node [block,left of = sensing,xshift=-2cm] (g) {$f(\vect{p}_{t-1},\bm{\eta}_{t-1})$};
\node [block_transp, left of = g, xshift=1cm] (omega) { $\bm{\eta}_{t-1}$};      
\node [block, below of = g, yshift=1.5cm] (time_update_g) { $z^{-1}$};
\node [block, right of = sensing, xshift=1.5cm] (sensor dynamics) { $\vect{s}_t = \vect{s}_{t-1} + \vect{u}_{t-1}^*$};
\node [block, below of = sensor dynamics, yshift=1.2cm,xshift=1cm] (time_update_sensor) { $z^{-1}$};
\node [block, below of = sensing, xshift=-1.2cm, yshift=-1.2cm] (RF_likelihood) { $p(z_{\text{RF},t}|\vect{p}_t,\vect{s}_t)$};
\node [block, below of = sensing, xshift=1.2cm,yshift=-1.2cm] (visual_likelihood) { $p(\vect{z}_{c,t}|\vect{p}_t,\vect{s}_t)$};
\node [sum, below of = sensing, yshift = -4.5cm] (prod) { $\times$};
\node [block, below of = prod, yshift = 1.7cm] (likelihood) { $p(\vect{z}_t|\vect{p}_t,\vect{s}_t)$};
\node [block, below of = sensor dynamics, yshift = -3.8cm] (RBE) { RBE};
\node [block, above of = RBE, yshift=-0.3cm] (control) { $\argmin{ J(\vect{s}_t + \vect{u}_t) }$};

\draw [->] (omega.east) --  node{} (g.west);
\draw [->] (g.east)  
-- ++ (2,0) 
-- node[name=output_g]{ $\vect{p}_{t}$} (sensing.west);
\draw [->]  ($(g.east)+(0.5,0)$) |- (time_update_g.east) ;
\draw [->]  (time_update_g.west) -- ++ (-1,0) -- node[]{ $\vect{p}_{t-1}$} ++ (0,1.2) -- ($(g.west)-(0,0.3)$);
\draw [->] (sensor dynamics.west) -- node[xshift=-0.2cm]{$\vect{s}_{t}$} (sensing.east);
\draw [->] ($(sensor dynamics.west)+(-1,0)$) |- (time_update_sensor.west) ;
\draw [->] (time_update_sensor.east) -- ++ (0.3,0) -- node[]{ $\vect{s}_{t-1},\vect{u}_{t-1}^*$} ++ (0,1.8) -- ($(sensor dynamics.east)$);
\draw [->] ($(time_update_sensor.west) + (-0.5,0)$) -- ($(control.north)$) ;
\draw [->] ($(sensing.south) -(0.3,0)$) -- ++ (0,-2) -- ++ (-0.9,0) --node[]{ $z_{\text{RF},t}$} (RF_likelihood.north);
\draw [->] ($(sensing.south) +(0.3,0)$) -- ++ (0,-2) -- ++ (0.9,0) -- node[xshift=-0.7cm]{ $\vect{z}_{c,t}$} (visual_likelihood.north);
\draw [->] (RF_likelihood.south) |- node{} (prod.west);
\draw [->] (visual_likelihood.south) |- node{} (prod.east);
\draw [->] (prod.south) -- node{} (likelihood.north);
\draw [->] (likelihood.east) -- node{} (RBE.west);
\draw [->] (RBE.north) -- node[]{ $p(\vect{p}_{t}|\vect{z}_{1:t})$} (control.south);
\draw [->] ($(RBE.north)+(0,0.5)$) -- ++ (1,0) -- ++ (0,-1.05) -- (RBE.east);
\draw [->] ($(control.north)+(1,0)$) -- node[]{ $\vect{u}_t^*$}(time_update_sensor.south);

\begin{scope}[on background layer]
\draw [fill=lightgray,dashed]  ($(omega) + (-0.5,0.7)$) rectangle ($(time_update_g) + (2,-0.7)$);
\draw [fill=lightgray,dashed] ($(RF_likelihood) + (-1.2,0.8)$) rectangle ($(likelihood) + (2.3,-0.7)$);
\draw [fill=lightgray,dashed] ($(control.west) + (-0.2,0.8)$) rectangle ($(control.east) + (0.2,-0.8)$);
\draw [fill=lightgray,dashed] ($(sensor dynamics.west) + (-1.8,0.7)$) rectangle ($(time_update_sensor.east) + (0.5,-0.7)$);
\draw [fill=lightgray,dashed] ($(RBE.west) + (-1,0.7)$) rectangle ($(RBE.east) + (1,-0.7)$);
\end{scope}
\draw [dashed] ($(RF_likelihood) + (-1.4,5.1)$) rectangle ($(control) + (2.2,-4.6)$);

\node [block_transp, above of = g, yshift=-2cm] (target_annotation) {\textit{Target dynamics}};
\node [block_transp, below of = sensor dynamics, xshift=-1.7cm,yshift=0.25cm] (sensor_annotation) { \textit{Platform dynamics}};
\node [block_transp, below of = likelihood, yshift=2cm] (likeliood_annotation) {\textit{Bi-modal likelihood}};
\node [block_transp, below of = control, yshift=2cm,xshift=1cm] (control_annotation) { \textit{Control}};
\node [block_transp, below of = RBE, yshift=2cm] (RBE_annotation) { \textit{Probabilistic map}};
\node [block_transp, above of = sensor dynamics, yshift=-1.8cm] (sensor_annotation) { \textit{Sensing platform}};

\end{tikzpicture}
\caption{ 
Scheme of the proposed Probabilistic Radio-Visual Active Sensing algorithm.}
\label{fig:pipeline}
\vspace{-0.5cm}
\end{figure}
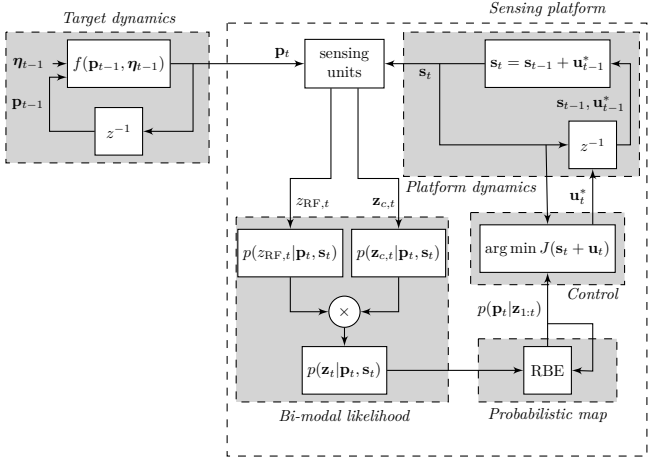

\textbf{Radio-visual likelihood - } 
The probabilistic form of an observation model is referred to as likelihood function. 
In active sensing frameworks, the likelihood accounts for both target and platform state,
that is $p(\vect{z}_t|\vect{p}_t,\vect{s}_t)$. 
In our case, 
the \emph{RF likelihood} is
\begin{equation}\label{eq:likelihood_RF}
p(z_{\text{RF},t}|\vect{p}_t,\vect{s}_t) =
\begin{cases}
\mathcal{N}\left(z| r_t,\sigma_{\text{RF}}^2 \right), & t=MT_{\text{RF}} \\
1, & \text{otherwise}.
\end{cases}
\end{equation}
Thus, $z_{\text{RF},t}$ updates the belief map only when it carries information on the target position (i.e., 
at $t=MT_{\text{RF}}$).

To define the visual likelihood from the observation model \eqref{eq:obs_c}, we consider the detection event as a Bernoulli random variable with success probability
\begin{equation}\label{eq:POD}
p(D_t=1|\vect{p}_t,\vect{s}_t)  =
\begin{cases}
\Upsilon(\vect{p}_t,\vect{s}_t), & \vect{p}_t \in \Phi(\vect{s}_t) \\
0, & \text{otherwise}
\end{cases}
\end{equation}
with $\Phi(\vect{s}_t)$ camera FoV onto $\Pi$ (see Fig. \ref{fig:scenario}), and
\begin{equation}\label{eq:Upsilon}
\Upsilon(\vect{p}_t,\vect{s}_t) =  \left[ 1 + e^{\gamma(d(\vect{c}_t,\vect{p}_t^+)/f - \epsilon)} \right]^{-1} \left( 1 + e^{-\gamma \epsilon} \right).
\end{equation}
From \eqref{eq:POD} the target can be detected only if inside the camera FoV and, from \eqref{eq:Upsilon} the detection probability is proportional to the resolution $d(\vect{c}_t,\vect{p}_t^+)/f$ at which the target is observed, where $f$ is the camera focal length; \mbox{$\epsilon > 0$} is the resolution at which the target is no longer well detectable, $\gamma > 0$ is the rate of the target detectability decrease. Then, the \textit{visual likelihood} is
\begin{equation}\label{eq:likelihood_V}
\begin{split}
& p(\vect{z}_{c,t}|\vect{p}_t,\vect{s}_t) = p(\vect{z}_{c,t}|\vect{p}_t,\vect{s}_t,D_t)p(D_t|\vect{p}_t,\vect{s}_t) =  \\
&
\begin{cases}
\mathcal{N}(\vect{z}| \vect{P}(\vect{s}_t)\widetilde{\vect{p}}_t,\bm{\Sigma}_c)\Upsilon(\vect{p}_t,\vect{s}_t), & D_t = 1,\; \vect{p}_t \in \Phi(\vect{s}_t) \\
0, & D_t = 1,\; \vect{p}_t \not\in \Phi(\vect{s}_t) \\
1-\Upsilon(\vect{p}_t,\vect{s}_t), & D_t = 0,\; \vect{p}_t \in \Phi(\vect{s}_t) \\
1, & D_t = 0,\; \vect{p}_t \not\in \Phi(\vect{s}_t) 
\end{cases}
\end{split}
\end{equation}

By aggregating radio and visual likelihoods, the following \emph{radio-visual likelihood} is obtained
\begin{equation}\label{eq:bimodal_likelihood_radioVisual} 
    p(\vect{z}_t |\vect{p}_t , \vect{s}_t ) = p(z_{\text{RF}.t} |\vect{p}_t ,\vect{s}_t )p(\vect{z}_{c,t} |\vect{p}_t , \vect{s}_t ).
\end{equation} 
Then, the update stage \eqref{eq:update} of the particle filter is applied using \eqref{eq:bimodal_likelihood_radioVisual}.

\noindent \textbf{Controller -} The platform control input is computed by solving the following optimization 
\begin{equation}\label{eq:control_input}
\mathcal{C}: \; \vect{u}_t^* =
\argmin_{\vect{u}_t \in \mathcal{A}} J(\vect{s}_t+\vect{u}_t), \text{ s.t. } d \left( \vect{c}_{t+1},\vect{c}_{t} \right) \leq E_t
\end{equation}
where $E_t$ is the residual energy at time $t$, computed as
\begin{equation}\label{eq:residual_energy}
    E_t = E_{tot} - \sum_{k=0}^{t-1} \Delta E_k = E_{tot} - \sum_{k=0}^{t-1} d(\vect{c}_{k+1},\vect{c}_k).
\end{equation}
The cost function is
$
J(\vect{c}_{\Pi,{t+1}}) = \frac{1}{2} d\left(\vect{c}_{\Pi,{t+1}}, \hat{\vect{p}}_t\right)^2,
$
where 
$\hat{\vect{p}}_t$ is the MAP estimate
of the target position.
Note that $J(\cdot)$ is function of $\vect{s}_t+\vect{u}_t$, since $\vect{c}_{\Pi,t+1}$ is related to $\vect{s}_{t+1}$ through the inverse perspective geometry~\cite{aghajan2009multi}.
Moreover, $J(\cdot)$ extracts information from the belief map, according to the probabilistic active sensing approach (Fig. \ref{fig:pipeline}).

The convexity of $J(\cdot)$ w.r.t. $\vect{c}_{\Pi,t+1}$ allows to solve \eqref{eq:control_input} with the gradient-based control law
\begin{equation}
\label{eq:gradient_based_controller}
\!\begin{cases}
\vect{s}_{\tau+1} = \vect{s}_\tau + \vect{u}_\tau, \; \tau \in [0, \tau_{max}] \\
\begin{split}
%
\\[-4pt]
\vect{u}_\tau 
    & = \!
    \begin{bmatrix} 
    \begin{array}{c}
    \vect{u}_{\vect{c},t} \\ 
    \hline
    \vect{u}_{\bm{\psi},t} \\[4pt]
    \end{array}
    \end{bmatrix}
    \!
    = 
    - \!\begin{bmatrix}
    \begin{array}{c|c}
    \vect{G}_{\vect{c}} & \vect{0} \\
    \hline
    \vect{0} & \vect{G}_{\bm{\psi}}
    \end{array}
    \end{bmatrix} 
    \!
    \begin{bmatrix} 
    \begin{array}{c}
    \frac{\partial J(\vect{c}_{\Pi,\tau+1})}{\partial \vect{c}_{\Pi}} \\ 
    0 \\
    \hline \\[-4pt]
\frac{\partial J(\vect{c}_{\Pi,\tau+1})}{\partial \bm{\psi}} \\[4pt] 
\end{array}\end{bmatrix}
\end{split}
\end{cases}
\end{equation}
where $\tau_{max}$ accounts for the maximum number of iterations in order to accommodate the next incoming measurement at $t+1$. 
$\vect{G}_{\vect{c}} \in \mathbb{R}^{3 \times 3}$ and $\vect{G}_{\bm{\psi}} \in \mathbb{R}^{2 \times 2}$ are suitable control gain matrices.
By choosing  $\vect{G}_{\vect{c}}$ entries small, energy is preserved, since \eqref{eq:gradient_based_controller} commands short UAV movements. Conversely, larger $\vect{G}_{\vect{c}}$ and $\vect{G}_{\bm{\psi}}$ lead to a more reactive system, capable of getting closer to the setpoint $\hat{\vect{p}}_t$ more quickly. 
It is important to remark that $J(\cdot)$ is \emph{purely-exploitative} and \emph{not energy-aware}: in $\mathcal{C}$ energy appears only in the constraint and no energy preservation~\cite{liu2018energy},
nor information-seeking (explorative)~\cite{radmard2017active} criteria are included.  


\section{Theoretical results }\label{sec:theoretical}


This Section formally motivates the use of an action space involving the entire camera pose, as in \eqref{eq:sensor_dynamics}, and supports the choice of a combined radio-visual perception system, as in \eqref{eq:bimodal_likelihood_radioVisual}. 
In uncluttered single-target scenarios the particle weight distribution is a possible 
indicator of the target localizability: highly-weighted regions 
allow to focus the position estimate, while uniform weight patterns suggest ambiguity in the target localization. 
In this respect, we show that radio-only solutions need the sensing platform to move in order to solve localization ambiguity (Ths.~1-2 that follow), which can be conversely attained through a radio-visual approach also with a static platform (Th.~3).

\begin{theorem}\label{th:axis_symmetric_ambiguity}
Let the following hypotheses hold
\begin{enumerate}
\item the target moves according to an unbiased random walk, i.e.,  \mbox{$p(\vect{p}_{t+1}|\vect{p}_t) = \mathcal{N}(\vect{p}|\vect{p}_t,\sigma^2 \vect{I}_2)$};
\item the platform is static, i.e. $\vect{c}_t = \vect{c}; \; \forall t$;
\item the RBE scheme updates through \eqref{eq:update} exploiting only the RF likelihood \eqref{eq:likelihood_RF}.
\end{enumerate}
Then, 
\begin{equation}\label{eq:axisSymmetric_condition}
\begin{split}
    & \mathbb{E} \left[ \omega_t^{(i)} | z_{\text{RF},1:t}\right] = \mathbb{E} \left[ \omega_t^{(j)} | z_{\text{RF},1:t}\right], \; t \geq 0 \\
    & \forall i,j \in [1,\dots,N_s] \text{ s.t. } d(\vect{p}_0^{(i)},\vect{c}_\Pi) = d(\vect{p}_0^{(j)},\vect{c}_\Pi)
\end{split}
\end{equation}
\end{theorem}
\begin{proof}
The dynamic model associated to the target unbiased random walk is 
\begin{equation}
    \vect{p}_{t+1} = \vect{p}_t + \bm{\eta}_t, \quad \bm{\eta}_t \sim \mathcal{N}(\bm{\eta}|\vect{0}_2,\sigma^2\vect{I}_2)
\end{equation}
Equivalently,
$
    \vect{p}_t = \vect{p}_0 + \sum_{k=0}^{t-1} \bm{\eta}_k, \; t>0.
$

\noindent
Given $\bm{\eta}_t, \; \bm{\eta}_s$ i.i.d. for any $s \neq t$, it follows
\begin{equation}
    \bar{\bm{\eta}}_{t-1} := \sum_{k=0}^{t-1}\bm{\eta}_k \sim \mathcal{N} \left(\bm{\eta} | \vect{0}_2,t\sigma^2\vect{I}_2 \right).
\end{equation}
Then, the squared distance $d_t^{(i),2} := d(\vect{p}_t^{(i)},\vect{c}_\Pi)^2$ is 
\begin{equation}
\begin{split}
    d_t^{(i),2} 
    & = \lVert \vect{p}_t^{(i)} - \vect{c}_\Pi \rVert_2^2  = \lVert \vect{p}_0^{(i)} + \bar{\bm{\eta}}_{t-1} - \vect{c}_\Pi \rVert_2^2 \\
    & = \sum_{\ell=1}^2 \left( p_0^{(i)}(\ell) - c_\Pi(\ell) \right)^2 + \sum_{\ell=1}^2 \bar{\eta}_{t-1}(\ell)^2  \\
    & + 2\sum_{\ell=1}^2 \left( p_0^{(i)}(\ell) - c_\Pi(\ell) \right) \bar{\eta}_{t-1}(\ell), \; t > 0.
\end{split}
\end{equation}
It holds,
\begin{equation}
\small
 2\left( p_0^{(i)}(\ell) - c_\Pi(\ell) \right)\! \bar{\eta}_{t-1}(\ell) \sim  \! \mathcal{N} \left( \! \eta | 0,4t\sigma^2\left( p_0^{(i)}(\ell) - c_\Pi(\ell) \right)^2 \!\right)
\end{equation}
and, since the components of $\bar{\bm{\eta}}_{t-1}$ are i.i.d. with distribution $\mathcal{N}\left(0,t\sigma^2\right)$,
$
    \sum_{\ell=1}^2 \bar{\eta}_{t-1}(\ell)^2 \sim t\sigma^2 \chi^2(2).
$
Finally, recalling that $\sum_{\ell=1}^2 ( p_0^{(i)}(\ell) - c_\Pi(\ell) )^2 = d_0^{(i)}$, we get
\begin{align}\label{eq:distance_distribution_th1}
    & d_t^{(i),2} \sim \mathcal{N}(d^2 |\mu_d, \sigma_d^2 ) + t\sigma^2 \chi^2(2), \; t \geq 0 \\
    & \mu_d = d_0^{(i),2}, \quad \sigma_d^2 = 4t\sigma^2d_0^{(i),2}.
\end{align}
From \eqref{eq:distance_distribution_th1}, $d_t^{(i),2}$ is identically distributed for all particles with the same initial distance from the platform. Given $z_{\text{RF},1:t}$, $\omega_t^{(i)}$ is a \mbox{(non-linear)} function of $d_t^{(i),2}$, through \eqref{eq:likelihood_RF} and \eqref{eq:path_loss_model}. Then, condition \eqref{eq:axisSymmetric_condition} follows.
\end{proof}

\begin{theorem}\label{th:nonStatic_Rx}
Let the following hypotheses hold
\begin{enumerate}
\item the target moves according to an unbiased random walk, i.e.,  \mbox{$p(\vect{p}_{t+1}|\vect{p}_t) = \mathcal{N}(\vect{p}|\vect{p}_t,\sigma^2 \vect{I}_2)$};
\item the platform moves at a fixed altitude according to a deterministic linear Markovian motion, i.e., $\vect{c}_{\Pi,t+1} = \vect{c}_{\Pi,t} + \vect{u}_{\vect{c}_\Pi,t}$;
\item the RBE scheme updates through \eqref{eq:update} exploiting only the RF likelihood \eqref{eq:likelihood_RF}.
\end{enumerate}
Then, 
\begin{equation}
\begin{split}
    & \exists (i,j) \in [1,N_s] \text{ s.t. } d(\vect{p}_0^{(i)},\vect{c}_{\Pi,0}) = d(\vect{p}_0^{(j)},\vect{c}_{\Pi,0})  \\
    & \text{ and } \mathbb{E} \left[ \omega_t^{(i)} | z_{\text{RF},1:t}\right] \neq \mathbb{E} \left[ \omega_t^{(j)} | z_{\text{RF},1:t}\right],\; t > 0
\end{split}
\end{equation}
\end{theorem}
\begin{proof}
The platform planar dynamics can be equivalently written as
\vspace{-0.2cm}
\begin{equation}
    \vect{c}_{\Pi,t}= \vect{c}_{\Pi,0} + \bar{\vect{u}}_{t-1}, \quad \bar{\vect{u}}_{t-1} := \sum_{k=0}^{t-1}\vect{u}_{\vect{c}_\Pi,k}, \; t>0.
\end{equation}
With similar computations involved of \mbox{Th. \ref{th:axis_symmetric_ambiguity}}, it is possible to show that the squared distance \mbox{$d_t^{(i),2} := d(\vect{p}_t^{(i)},\vect{c}_{\Pi,t})^2$} is distributed as 
\begin{equation}\label{eq:distance_distribution_th2}
    d_t^{(i),2} \sim
    \mathcal{N}(d^2 |\mu_d,\sigma_d^2 ) +  t\sigma^2  \chi^2(2,\lambda), \; t>0
\end{equation}
with
\begin{equation}\label{eq:parameters_distance_th2}
\begin{split}
    & \mu_d = d_0^{(i),2} - 2(\vect{p}_0^{(i)} - \vect{c}_{\Pi,0})^\top \bar{\vect{u}}_{t-1}, \;\; \sigma_d^2 = 4t\sigma^2 d_0^{(i),2}\\
    & \lambda = -\sum_{\ell=1}^2 \bar{u}_{\vect{c}_\Pi,t-1}(\ell)^2.
\end{split}
\end{equation}
$\chi^2(2,\lambda)$ is a non-central chi-squared distribution.
From \mbox{\eqref{eq:distance_distribution_th2}-\eqref{eq:parameters_distance_th2}}, $d_t^{(i),2}$ depends on both $d_0^{(i),2}$ and \mbox{$(\vect{p}_0^{(i)} - \vect{c}_{\Pi,0})^\top \bar{\vect{u}}_{t-1}$}. It is always possible to find $(i,j) \in [1,N_s]$, such that $d_0^{(i)}=d_0^{(j)}$ and \mbox{$(\vect{p}_0^{(i)} \!-\! \vect{c}_{\Pi,0})^\top \bar{\vect{u}}_{t-1} \! \neq \! (\vect{p}_0^{(j)} \! - \! \vect{c}_{\Pi,0})^\top \bar{\vect{u}}_{t-1}$}. Then, in general, 
$
 \mathbb{E} \left[ \omega_t^{(i)} | z_{\text{RF},1:t}\right] \neq \mathbb{E} \left[ \omega_t^{(j)} | z_{\text{RF},1:t}\right]. 
$
\end{proof}

\begin{theorem}\label{th:biModality_benefit}
Let the following hypotheses hold
\begin{enumerate}
\item the target moves according to an unbiased random walk, i.e.,  \mbox{$p(\vect{p}_{t+1}|\vect{p}_t) = \mathcal{N}(\vect{p}|\vect{p}_t,\sigma^2\vect{I}_2)$};
\item the platform is static, i.e. $\vect{c}_t = \vect{c}; \; \forall t$;
\item the RBE scheme updates through \eqref{eq:update} exploiting the radio-visual likelihood \eqref{eq:bimodal_likelihood_radioVisual}.
\end{enumerate}
Then, 
\begin{equation}\label{eq:non_axisSymmetric_condition}
\begin{split}
    & \exists (i,j) \in [1,N_s] \text{ s.t. } d(\vect{p}_0^{(i)},\vect{c}_{\Pi}) = d(\vect{p}_0^{(j)},\vect{c}_{\Pi}) \text{ and } \\
    & \mathbb{E} \left[ \omega_t^{(i)} | z_{\text{RF},1:t}\right] \neq \mathbb{E} \left[ \omega_t^{(j)} | z_{\text{RF},1:t}\right], \; t>0
\end{split}
\end{equation}
\end{theorem}
\begin{proof}
Suppose (w.l.o.g.) that $\vect{p}_0^{(i)}$ and $\vect{p}_0^{(j)}$ satisfy \mbox{$d(\vect{p}_0^{(i)},\vect{c}_{\Pi}) = d(\vect{p}_0^{(j)},\vect{c}_{\Pi})$}; hence, $d_t^{(i),2}$ and $d_t^{(j),2}$ are equally distributed, from Th. \ref{th:axis_symmetric_ambiguity}. Suppose now that \mbox{$\vect{p}_t^{(i)} \in \Phi(\vect{s}_t)$}, but \mbox{$\vect{p}_t^{(j)} \not\in \Phi(\vect{s}_t)$}. From \eqref{eq:likelihood_V}-\eqref{eq:bimodal_likelihood_radioVisual}, $\omega_t^{(i)}$ and $\omega_t^{(j)}$ are different functions of two equally distributed random variables (once $z_{\text{RF},1:t}$ are fixed). Thus, in general,
$
    \mathbb{E} \left[ \omega_t^{(i)} | z_{\text{RF},1:t}\right] \neq \mathbb{E} \left[ \omega_t^{(j)} | z_{\text{RF},1:t}\right].
$
\end{proof}



\begin{figure*}[htp!]
\vspace{0.2cm}
\centering
\subfigure[]{
\includegraphics[width=0.3\textwidth]{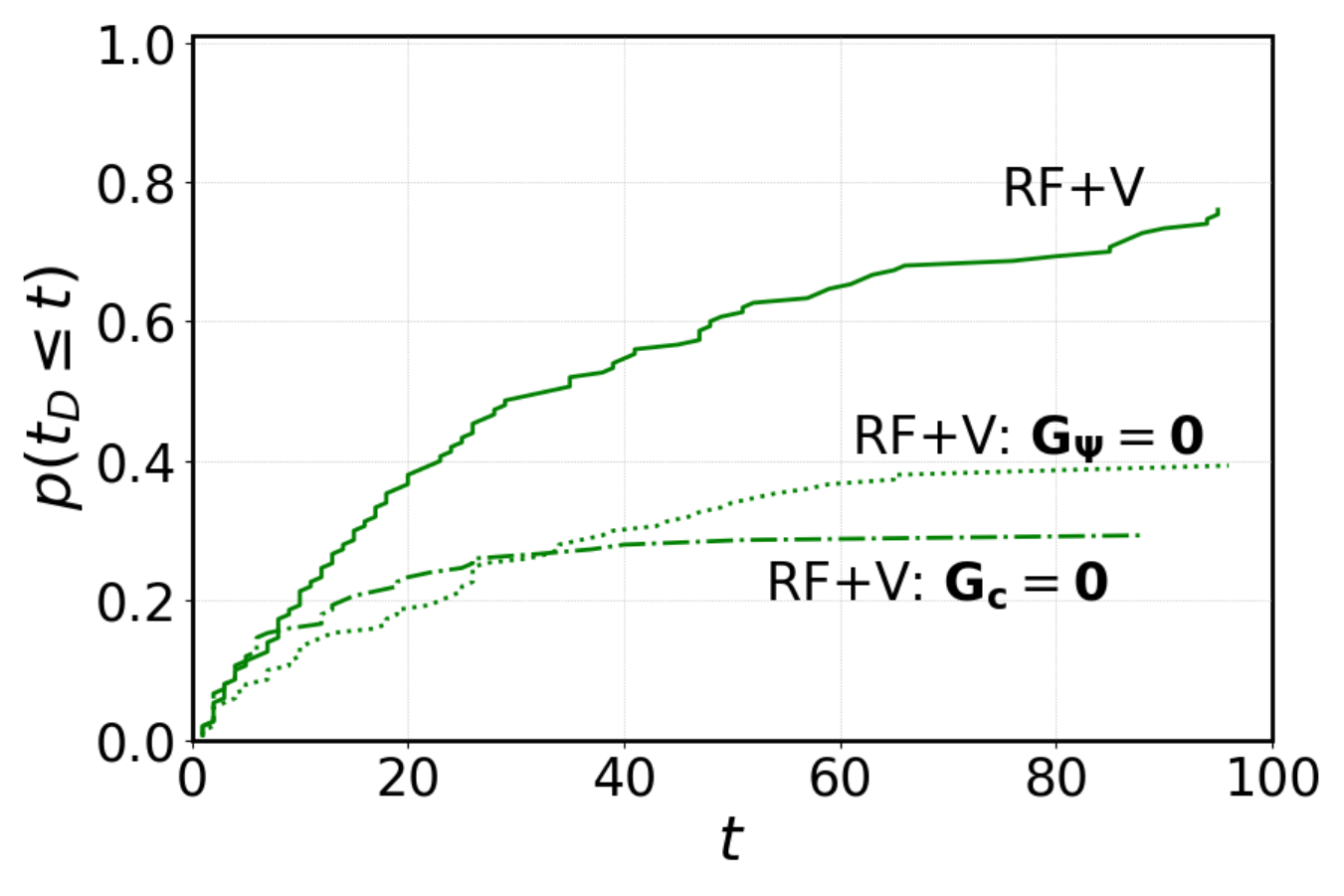}
\label{fig:controlSpace}
}
\subfigure[]{
\includegraphics[width=0.3\textwidth]{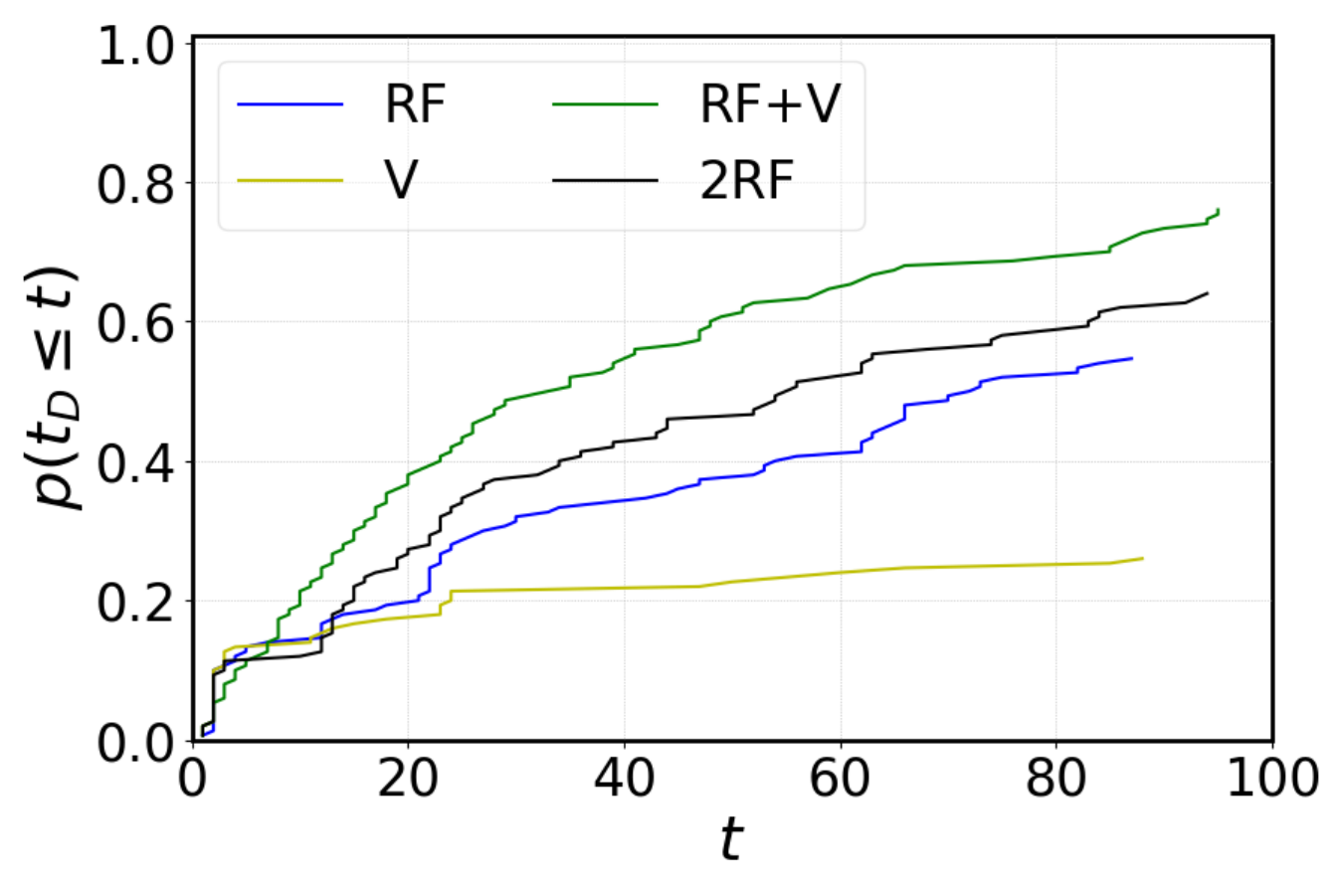}
\label{fig:modalities}
}
\subfigure[]{
\includegraphics[width=0.3\textwidth]{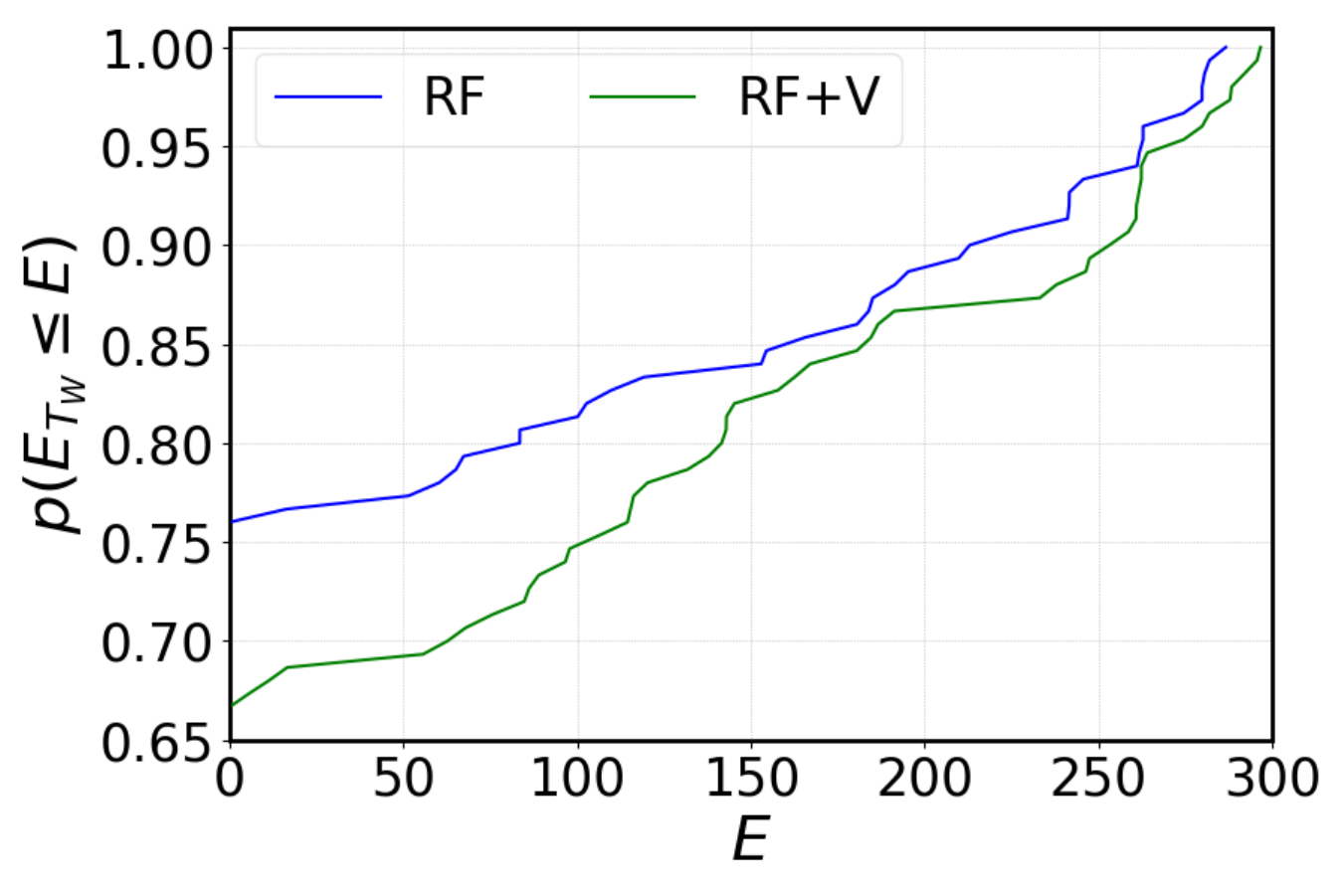}
\label{fig:energy}
}
\vspace{-0.2cm}
\caption{MC experiment results. (a) Detection time ECDF with different control spaces: full camera pose (RF+V), no pan-tilt movements (RF+V: $\vect{G}_{\bm{\Psi}}=\vect{0}$), static platform (RF+V: $\vect{G}_{\vect{c}}=\vect{0}$). (b) Detection time ECDF with different type and amount of information employed for the tracking task: radio-visual (RF+V), visual-only (V), radio-only (RF), and radio-only with two receivers ($2$RF). (c) Energy consumption ECDF between RF+V and RF.}
\label{fig:numerical_results}
\vspace{-0.5cm}
\end{figure*}

\textbf{Discussion - }
Th. \ref{th:axis_symmetric_ambiguity} statistically characterizes the {\em axis-symmetric ambiguity}, one of the main issues in RSSI-based localization. As \eqref{eq:path_loss_model} suggests, RSSI values bring information only on the target-receiver distance; consequently, 
if the receiver is static, the belief map is toroidal with non-unique MAP estimate and severe convergence issues may arise.
The axis-symmetric effect can be mitigated by using a moving receiver (Th. \ref{th:nonStatic_Rx}), or by aggregating visual information to RSSI data (Th. \ref{th:biModality_benefit}). In particular, to completely exploit the disambiguation effect of visual observations, camera orientation should be included in the control space, as in \eqref{eq:sensor_dynamics}. 

In conclusion, combining radio-visual cues in the RBE scheme, and comprising the entire camera pose in $\mathcal{A}$, reduces the estimation ambiguities. This makes the localization procedure faster and more accurate. 
Accordingly, the overall visual detection is more robust and more efficient (see \mbox{Sec. \ref{sec:numerical}}).

\section{Numerical results}\label{sec:numerical}

To support the theoretical results, the control law
$\mathcal{C}$ is used to track a moving target in a Python-based synthetic environment.
At first, we numerically motivate the choice of the entire camera pose as control space. Secondly, we show that bi-modality allows to achieve higher robustness and time-efficiency than several uni-modal baselines. Finally, the proposed algorithm is proven to be even more energy-efficient than a radio-only counterpart, albeit $\mathcal{C}$ does not account for any energy-preserving term.  

\textbf{Setup parameters - }
To capture the performance variability, numerical evaluation is performed through a Monte Carlo (MC) experiment, with $N_{tests} = 150$ tests of $T_W=100$ iterations each.
The underlying target motion is simulated using a linear stochastic Markovian transition planar model
\begin{equation}
 \vect{p}_{t+1} = \vect{p}_{t} + \vect{q}_{t}; \quad
 \vect{q}_{t} \sim \mathcal{N}(\vect{q}|\bm{\mu}_{\vect{q}},\bm{\Sigma}_{\vect{q}}), \; \bm{\mu}_{\vect{q}} \neq \vect{0}_2.
\end{equation}
The initial condition $\vect{p}_{0}$ is randomly changed at each MC test, as well as the platform initial planar position $\vect{c}_{\Pi,0}$. The process model is an unbiased random walk, the receiver sampling rate is $T_{\text{RF}}=10$ ($T_a=1$), the total available energy is fixed to $E_{tot}=300$, and the RSSI noise level is $\sigma_{\text{RF}}=3.5 \; [dBm]$.

\textbf{Performance assessment - }
The following performance indexes are computed from the MC simulation, recalling the definition of Empirical Cumulative Distribution Function (ECDF):
    \begin{equation}\label{eq:ecdf}
    p(q \! \leq \! Q) \!=\! \frac{1}{N_{tests}} \! \sum_{j=1}^{N_{tests}} \! \mathds{1}_{q_j\leq Q}, \;
    \mathds{1}_{q_j\leq Q} \!=\! 
    \begin{cases}
    1, \! & \! \text{if } q_j \! \leq \! Q  \\
    0, \! & \! \text{otherwise}
    \end{cases}
    \end{equation}
    where $q_j$ is the generic variable $q$ at the $j$-th MC test.

\textit{Detection success rate}: rate of successful target visual detections over the MC tests (i.e., robustness index)
\begin{equation}\label{eq:detection_success_rate}
\begin{split}
& \varpi = \frac{1}{N_{tests}}\sum_{j=1}^{N_{tests}} \mathds{1}_{D,j} \\
& 
\mathds{1}_{D,j} = 
\begin{cases}
1, & \text{if } \exists t_{D} \in [1,T_W] \; \text{s.t. } D_{t_{D}}=1  \\
0, & \text{otherwise}
\end{cases}
\end{split}
\end{equation}

\textit{Detection time ECDF}: time-efficiency index accounting for the time, $t_D$, before the target is visually detected. Its ECDF is obtained from \eqref{eq:ecdf} with $q=t_D$ and $Q=t$.

\textit{Energy consumption ECDF}: energy-efficiency index accounting for the available energy at the end of the task, i.e., \mbox{$E_{T_W} \in [0,E_{tot}]$}. Its ECDF is obtained from \eqref{eq:ecdf} with $q=E_{T_W}$ and $Q=E$.


\textbf{Impact of the control space - }
The proposed bi-modal radio-visual approach (RF+V) is compared with two variants: the first fixes the camera downwards (i.e., \mbox{$\vect{G}_{\bm{\Psi}}=\vect{0}$} and $\alpha_t=\beta_t=0, \; \forall t$); the second considers a static platform (i.e., $\vect{G}_{\vect{c}}=\vect{0}$). Without pan-tilt actuation, the target is inside the FoV only when underneath the UAV.
However, this requirement is difficult to be met for every initial condition $\vect{p}_0$ and $\vect{c}_0$, due to localization errors and to the bias in the target motion (i.e., $\bm{\mu}_{\vect{q}}$). 
On the other side, a static platform is not capable to reduce its distance w.r.t. the target; hence, many detection failures may occur even when the target is in FoV, according to \eqref{eq:POD}-\eqref{eq:Upsilon}.
Indeed, Tab.~\ref{tab:detection_rate} (top row) shows that the detection success rate of RF+V is $77 \%$, against $39 \%$ and $29 \%$ of the cases $\vect{G}_{\bm{\Psi}}=\vect{0}$ and $\vect{G}_{\vect{c}}=\vect{0}$, respectively. Moreover, Fig. \ref{fig:controlSpace} shows a better detection time ECDF of RF+V w.r.t the other two versions. 
In conclusion, including the entire camera pose in the control space brings more robustness and more time-efficiency.


\begin{table}[t]
\centering
\caption{Detection success rate of the compared approaches.}
\label{tab:detection_rate}
\begin{tabular}{c| c|c|c}
\hline
\rowcolor{lightgray}
 & RF+V & RF+V: $\vect{G}_{\bm{\Psi}}=\vect{0}$ & RF+V:  $\vect{G}_{\vect{c}}=\vect{0}$ \\
$\varpi$ &  $77\%$ & $39\%$ & $29\%$  \\
\hline
\rowcolor{lightgray}
& RF & V & 2RF  \\
$\varpi$ &  $55\%$ & $26\%$ & $64\%$    \\
\hline
\end{tabular}
\vspace{-0.5cm}
\end{table}

\textbf{Impact of the sensing modalities - }
Here RF+V is compared with a visual-only (V) and a radio-only (RF) variant.
According to Tab. \ref{tab:detection_rate} (bottom row) and Fig. \ref{fig:modalities}, RF+V is the most robust and time-efficient among the three algorithms. Indeed,
bi-modality induces higher localization accuracy with faster convergence rates. Consequently, the platform and the FoV are quickly driven towards the target, which is fundamental to have successful visual detection, according to \eqref{eq:POD}-\eqref{eq:Upsilon}.

One may argue that the superior performance of RF+V is only due to the larger amount of information involved. Thus, we compare RF+V with $2$RF: instead of combining radio-visual data, we aggregate observations from two receivers. 
The RSSI sample from the \mbox{$\ell$-th} receiver is $z_{\text{RF}^{(\ell)},t}$. 
Supposing independence between $z_{\text{RF}^{(1)},t}$ and $z_{\text{RF}^{(2)},t}$, the likelihood of $2$RF becomes
\vspace{-0.2cm}
\begin{equation}
p(\vect{z}_{2\text{RF},t} |\vect{p}_t , \vect{s}_t )  = \prod_{\ell=1}^2  p(z_{\text{RF}^{(\ell)},t}|\vect{p}_t, \vect{s}_t),
\end{equation}
with $p(z_{\text{RF}^{(\ell)},t}|\vect{p}_t, \vect{s}_t)$ as in \eqref{eq:likelihood_RF}.
Tab. \ref{tab:detection_rate} and Fig. \ref{fig:modalities} show that $2$RF improves RF, but RF+V is still better. This means that 
exploiting complementary cues (e.g. radio and visual ones) is more advantageous than combining homogeneous data extracted from different sources; in fact, one can prove that the axis-symmetric ambiguity is not reduced when applying multiple receivers placed at the same location.


\textbf{Energy-efficiency - }
As mentioned in Sec.~\ref{sec:method}, the proposed control law~\eqref{eq:control_input} does not include any explicit energy-aware term; hence, $\mathcal{C}$ is not expected to generate energy-preserving platform movements (unless $\vect{G}_{\vect{c}}$ is kept small, but this has been proven to be ineffective). Therefore, when the total available energy runs out, i.e.
\begin{equation}\label{eq:out_of_energy_condition}
\exists t \in [1,T_W] \text{ s.t. } E_{t}=0,     
\end{equation}
the platform becomes static and the detection capabilities dramatically decrease, as in \mbox{Fig. \ref{fig:controlSpace}}. In the absence of an energy-aware control technique, the only way to preserve energy is by producing accurate target position estimates: if $\hat{\vect{p}}_t \approx \vect{p}_t$, the controller $\mathcal{C}$ drives the platform towards the target through a smooth and direct trajectory, which is more energy-efficient than irregular patterns, according to \eqref{eq:energy_model}. Consequently, the energy-efficiency of an energy-agnostic tracking algorithm is an indirect measure of its localization accuracy. In this regard, \mbox{Fig. \ref{fig:energy}} shows the superiority of RF+V w.r.t. RF. More specifically, the two ECDFs satisfy
\begin{equation}
    p(E_{T_W}^{\text{RF+V}} \leq E) < p(E_{T_W}^{\text{RF}} \leq E);\; \forall E \in [0,E_{tot}]
\end{equation}
and the out-of-energy condition \eqref{eq:out_of_energy_condition}
is higher in RF ($76\%$) than in RF+V ($68\%$).

\section{Conclusion }\label{sec:conclusion}

This work proposes a probabilistic radio-visual active sensing scheme for RF-emitting target search. The suggested approached is supported by a theoretical analysis and validated via numerical experiments. These highlight the benefits of bi-modality in terms of robustness, as well as time and energy efficiency.





{\renewcommand{\baselinestretch}{0.988}
\bibliographystyle{IEEEtran}
\bibliography{IEEEfull,References}

\begin{thebibliography}{10}
\providecommand{\url}[1]{#1}
\csname url@samestyle\endcsname
\providecommand{\newblock}{\relax}
\providecommand{\bibinfo}[2]{#2}
\providecommand{\BIBentrySTDinterwordspacing}{\spaceskip=0pt\relax}
\providecommand{\BIBentryALTinterwordstretchfactor}{4}
\providecommand{\BIBentryALTinterwordspacing}{\spaceskip=\fontdimen2\font plus
\BIBentryALTinterwordstretchfactor\fontdimen3\font minus
  \fontdimen4\font\relax}
\providecommand{\BIBforeignlanguage}[2]{{%
\expandafter\ifx\csname l@#1\endcsname\relax
\typeout{** WARNING: IEEEtran.bst: No hyphenation pattern has been}%
\typeout{** loaded for the language `#1'. Using the pattern for}%
\typeout{** the default language instead.}%
\else
\language=\csname l@#1\endcsname
\fi
#2}}
\providecommand{\BIBdecl}{\relax}
\BIBdecl

\bibitem{radmard2017active}
S.~Radmard and E.~A. Croft, ``Active target search for high dimensional robotic
  systems,'' \emph{Autonomous Robots}, vol.~41, no.~1, pp. 163--180, 2017.

\bibitem{MTS_Rinner}
S.~P{\'e}rez-Carabaza, J.~Scherer, B.~Rinner, J.~A. L{\'o}pez-Orozco, and
  E.~Besada-Portas, ``U{A}{V} trajectory optimization for minimum time search
  with communication constraints and collision avoidance,'' \emph{Engineering
  Applications of Artificial Intelligence}, vol.~85, pp. 357--371, 2019.

\bibitem{shahidian2017single}
S.~A.~A. Shahidian and H.~Soltanizadeh, ``Single-and multi-{U}{A}{V} trajectory
  control in {R}{F} source localization,'' \emph{Arabian Journal for Science
  and Engineering}, vol.~42, no.~2, pp. 459--466, 2017.

\bibitem{detection_tracking_survey}
C.~Robin and S.~Lacroix, ``Multi-robot target detection and tracking: taxonomy
  and survey,'' \emph{Autonomous Robots}, vol.~40, no.~4, pp. 729--760, 2016.

\bibitem{smith2013MonteCarlo}
A.~Smith, \emph{Sequential Monte Carlo methods in practice}.\hskip 1em plus
  0.5em minus 0.4em\relax Springer Science \& Business Media, 2013.

\bibitem{hasanzade2018rf}
M.~Hasanzade, {\"O}.~Hereko{\u{g}}lu, R.~Yeni{\c{c}}eri, E.~Koyuncu, and
  G.~{\.I}nalhan, ``R{F} source localization using unmanned aerial vehicle with
  particle filter,'' in \emph{2018 9th International Conference on Mechanical
  and Aerospace Engineering (ICMAE)}.\hskip 1em plus 0.5em minus 0.4em\relax
  IEEE, 2018, pp. 284--289.

\bibitem{negative_information}
W.~Koch, ``On exploiting ‘negative’sensor evidence for target tracking and
  sensor data fusion,'' \emph{Information Fusion}, vol.~8, no.~1, pp. 28--39,
  2007.

\bibitem{liu2018energy}
C.~H. Liu, Z.~Chen, J.~Tang, J.~Xu, and C.~Piao, ``Energy-efficient {U}{A}{V}
  control for effective and fair communication coverage: A deep reinforcement
  learning approach,'' \emph{IEEE Journal on Selected Areas in Communications},
  vol.~36, no.~9, pp. 2059--2070, 2018.

\bibitem{haubner2019active}
T.~Haubner, A.~Schmidt, and W.~Kellermann, ``Active acoustic source tracking
  exploiting particle filtering and {Monte} {Carlo} tree search,'' in
  \emph{2019 27th European Signal Processing Conference (EUSIPCO)}.\hskip 1em
  plus 0.5em minus 0.4em\relax IEEE, 2019, pp. 1--5.

\bibitem{aghajan2009multi}
H.~Aghajan and A.~Cavallaro, \emph{Multi-camera networks: principles and
  applications}.\hskip 1em plus 0.5em minus 0.4em\relax Academic press, 2009.

\bibitem{SAR_radio}
S.~{\'O}. Murphy, C.~Sreenan, and K.~N. Brown, ``Autonomous unmanned aerial
  vehicle for search and rescue using software defined radio,'' in \emph{2019
  IEEE 89th Vehicular Technology Conference (VTC2019-Spring)}.\hskip 1em plus
  0.5em minus 0.4em\relax IEEE, 2019, pp. 1--6.

\bibitem{zafari2019survey}
F.~Zafari, A.~Gkelias, and K.~K. Leung, ``A survey of indoor localization
  systems and technologies,'' \emph{IEEE Communications Surveys \& Tutorials},
  vol.~21, no.~3, pp. 2568--2599, 2019.

\bibitem{zanella2016best}
A.~Zanella, ``Best practice in {R}{S}{S} measurements and ranging,'' \emph{IEEE
  Communications Surveys \& Tutorials}, vol.~18, no.~4, pp. 2662--2686, 2016.

\bibitem{DeepRL_gazeControl}
S.~Lathuili{\`e}re, B.~Mass{\'e}, P.~Mesejo, and R.~Horaud, ``Neural network
  based reinforcement learning for audio--visual gaze control in human--robot
  interaction,'' \emph{Pattern Recognition Letters}, vol. 118, pp. 61--71,
  2019.

\bibitem{yolo}
J.~Redmon, S.~Divvala, R.~Girshick, and A.~Farhadi, ``You only look once:
  Unified, real-time object detection,'' in \emph{Proceedings of the IEEE
  Conference on Computer Vision and Pattern Recognition}, 2016, pp. 779--788.

\bibitem{zhuang2016smartphone}
Y.~Zhuang, J.~Yang, Y.~Li, L.~Qi, and N.~El-Sheimy, ``Smartphone-based indoor
  localization with bluetooth low energy beacons,'' \emph{Sensors}, vol.~16,
  no.~5, p. 596, 2016.

\bibitem{vollmer2011high}
M.~Vollmer and K.-P. M{\"o}llmann, ``High speed and slow motion: the technology
  of modern high speed cameras,'' \emph{Physics Education}, vol.~46, no.~2, p.
  191, 2011.

\end{thebibliography}
}
\vfill
\end{document}